\def\year{2021}\relax
\documentclass[letterpaper]{article} % DO NOT CHANGE THIS
\usepackage{aaai21}  % DO NOT CHANGE THIS
\usepackage{times}  % DO NOT CHANGE THIS
\usepackage{helvet} % DO NOT CHANGE THIS
\usepackage{courier}  % DO NOT CHANGE THIS
\usepackage[hyphens]{url}  % DO NOT CHANGE THIS
\usepackage{graphicx} % DO NOT CHANGE THIS
\usepackage{natbib}  % DO NOT CHANGE THIS AND DO NOT ADD ANY OPTIONS TO IT
\usepackage{caption} % DO NOT CHANGE THIS AND DO NOT ADD ANY OPTIONS TO IT
\usepackage{multirow}

\usepackage[skip=3pt,labelformat=simple]{subcaption}
\usepackage{enumerate}
\usepackage{enumitem}
\setlist{nosep}

\title{A Competitive Analysis of Online Multi-Agent Path Finding\thanks{This work was supported by the Natural Sciences and Engineering Research Council (NSERC) under grant number RGPIN-2020-06540.}
}
\author{Hang Ma\\}
\affiliations{Simon Fraser University\\ hangma@sfu.ca}

\usepackage{amsfonts} 
\usepackage{amsmath}
\usepackage{amsthm}
\usepackage{amssymb}  % for \varnothing

\usepackage{tikz}

\theoremstyle{plain}
\newtheorem{thm}{Theorem}
\newtheorem{cor}[thm]{Corollary}

\newtheorem{obs}[thm]{Observation}
\theoremstyle{definition}
\newtheorem{dfn}[thm]{Definition}

\makeatletter
\renewenvironment{proof}[1][\proofname]{\par
	\vspace{-.5\topsep}% remove the space after the theorem
	\pushQED{\qed}%
	\normalfont
	\topsep0pt \partopsep0pt % no space before
	\trivlist
	\item[\hskip\labelsep
	\itshape
	#1\@addpunct{.}]\ignorespaces
}{%
	\popQED\endtrivlist\@endpefalse
	\addvspace{6pt plus 6pt} % some space after
}
\makeatother

\usepackage[linesnumbered,vlined,ruled]{algorithm2e}
\SetKwProg{Fn}{Function}{}{}
\SetAlFnt{\small}
\SetAlCapFnt{\normalfont\small}
\SetAlCapNameFnt{\small}
\SetArgSty{textrm}
\SetInd{0.1em}{0.5em}
\SetCommentSty{emph}

\allowdisplaybreaks
\usepackage{rotating}
\newcommand*\rot{\rotatebox[origin=c]{90}}
\DeclareMathOperator*{\argmax}{arg\,max}

\usepackage{fancyhdr}
\fancypagestyle{empty}{%
	\fancyhf{}%
	\cfoot{\textbf{This version of the paper is intended to update the version published in the Proceedings of the Thirty-First International Conference on Automated Planning and Scheduling (ICASP 2021). The definition of \textit{makespan} (Section \ref{sec:definition}) has been corrected from \pmb{$\max_{a_i\in\mathcal A} (t^{(g)}_i - r_i)$} to \pmb{$\max_{a_i\in\mathcal A} t^{(g)}_i$}. Other parts of the paper are not affected.}}
}
\begin{document}
	
	\maketitle
	
	\begin{abstract}
		We study online Multi-Agent Path Finding (MAPF), where new agents are constantly revealed over time and all agents must find collision-free paths to their given goal locations. We generalize existing complexity results of (offline) MAPF to online MAPF. We classify online MAPF algorithms into different categories based on (1) controllability (the set of agents that they can plan paths for at each time) and (2) rationality (the quality of paths they plan) and study the relationships between them. We perform a competitive analysis for each category of online MAPF algorithms with respect to commonly-used objective functions. We show that a naive algorithm that routes newly-revealed agents one at a time in sequence achieves a competitive ratio that is asymptotically bounded from both below and above by the number of agents with respect to flowtime and makespan. We then show a counter-intuitive result that, if rerouting of previously-revealed agents is not allowed, any rational online MAPF algorithms, including ones that plan optimal paths for all newly-revealed agents, have the same asymptotic competitive ratio as the naive algorithm, even on 2D 4-neighbor grids. We also derive constant lower bounds on the competitive ratio of any rational online MAPF algorithms that allow rerouting. The results thus provide theoretical insights into the effectiveness of using MAPF algorithms in an online setting for the first time.
		
	\end{abstract}
	
\section{Introduction}

Online Multi-Agent Path Finding (MAPF) \cite{vsvancara2019online} models the problem of finding collision-free paths for a stream of incoming agents in a given region. Its applications include autonomous intersection management \cite{dresner2008multiagent}, UAV traffic management \cite{ho2019multi}, video games \cite{MaAIIDE17}, and automated warehouse systems \cite{kiva}.
For example, Figure \ref{fig:kiva} shows the typical grid layout of part of a modern automated warehouse, where warehouse robots (orange squares) need to move inventory pods between their storage locations (green cells) and inventory stations (squares in purple and pink).
%For example, Figure \ref{fig:kiva} shows the typical grid layout of part of a modern automated warehouse with inventory stations, each with an entrance (purple cells) and an exit (pink cells), on the left side and storage locations (green cells), each being able to store one inventory pod, on the right side.
%%Each inventory pod consists of a stack of trays, each of which holds bins with products.
%The warehouse robots (orange squares) need to move inventory pods from their storage locations to inventory stations or vice versa. The narrow corridors in the storage region are single-direction lanes,
%%assigned alternate directions in which the robots must follow. where path planning is done by a traffic-rule-based system.
%where traffic is controlled by a traffic-rule-based system.
The narrow corridors in the storage region are single-direction lanes where traffic is controlled by a traffic-rule-based system.
However, paths are not known and must be planned for warehouse robots in the intersection region (red rectangle)
%between the inventory stations and storage locations)
that is often highly congested. Warehouse robots constantly enter this region at their given start cells and must plan collision-free paths to and exit at their given goal cells (along an edge of the rectangle).
The problem in such applications is online because each agent is known only when it is about to enter such a region and the future arrivals of agents are not known ahead of time.

Existing research has conducted empirical evaluations of several online MAPF algorithms \cite{vsvancara2019online} based on recent techniques for (offline) MAPF \cite{SternSOCS19}. However, there is still a lack of theoretical understanding of online MAPF and its algorithms. In this paper, we thus perform a theoretical analysis
%of online MAPF
from the points of view of competitive analysis and complexity theory.

\begin{figure}\centering
	\includegraphics[width=.63\columnwidth]{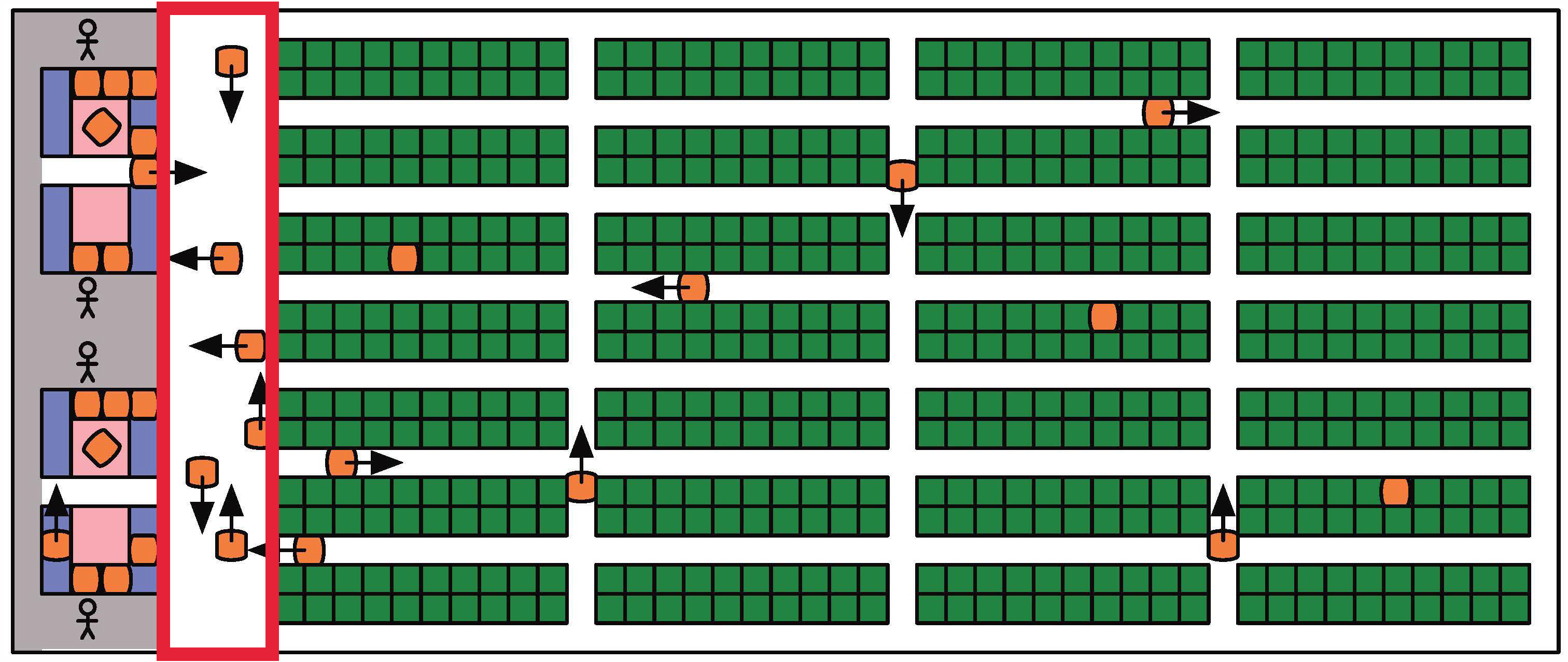}
	\caption{The 2D grid layout of part of an Amazon Robotics automated warehouse, reproduced from \citet{kiva}.}
	\label{fig:kiva}
\end{figure}

\subsection{Related Work}

\noindent\textbf{Offline MAPF:} Online MAPF is an extension of the well-studied problem of (offline) MAPF \cite{MaAIMATTERS17,SternSOCS19}, where all agents are known and start routing at the same time. MAPF is NP-hard to solve optimally for flowtime (the sum of the arrival times of all agents at their goal locations) minimization and to approximate within any constant factor less than 4/3 for makespan (the maximum of the arrival times of all agents at their goal locations) minimization \cite{surynek2010optimization,YuLav13AAAI,MaAAAI16}. It is NP-hard to solve optimally even on planar graphs \cite{yu2015intractability} and 2D 4-neighbor grids \cite{banfi2017intractability}. MAPF algorithms include reductions to other combinatorial problems \cite{YuLav13ICRA,erdem2013general,DBLP:conf/ecai/SurynekFSB16} and specialized algorithms \cite{PushAndSwap,Wang11,DBLP:journals/ai/SharonSGF13,DBLP:journals/ai/SharonSFS15,ICBS,CohenIJCAI18,MaAAAI19a,LiIJCAI19,LiICAPS19,LamBHS19,GangeHS19,LiICAPS20}.

\noindent\textbf{Online Problems:} \citet{MaAAMAS17} and \citet{MaAAAI19b} have considered an online version of MAPF where a given set of agents must attend to a stream of tasks, consisting of (sub-)goal locations to be assigned to the agents, that appear at unknown times. This version considers the entire environment instead of a region of a system and thus does not consider the appearance and disappearance of agents. \citet{vsvancara2019online} and \citet{ho2019multi} have considered another online version of MAPF, similar to the setting of this paper, where a stream of agents with preassigned goal location appear at unknown times. Algorithms for solving such online problems reduce each problem to a sequence of (offline) MAPF sub-problems that are solved by a MAPF algorithm. The effectiveness of these algorithms is characterized by objective functions that measure how soon the tasks are finished or the agents are routed to their goal locations. Existing study on online versions of MAPF has been empirical only. For example, both \citet{MaAAMAS17} and \citet{vsvancara2019online} have experimentally shown that algorithms that allow agents (that have paths already) to replan their paths and reroute tend to be more effective than those that do not. However, there is still a lack of theoretical understanding of solving MAPF in an online setting.

\subsection{Assumptions and Contributions}

We follow most of the notations of \citet{vsvancara2019online} and consider the setting where new agents can wait infinitely long before entering a given region and agents disappear upon exiting from the region because (1) existing online MAPF algorithms have been designed and tested only for this setting \cite{ho2019multi,vsvancara2019online}, although other settings concerning what happens before agents enter the region and after agents leave the region have (only) been mentioned (briefly) by \citet{SternSOCS19,vsvancara2019online} and (2) queuing at entrances and exits of such an intersection region is handled by a task-level planner/scheduler with reserved queuing spaces (for example, queues in inventory stations and along the single-lane corridors in the storage region) in automated warehouses \cite{kiva,kou2020idle} and many other real-world systems.
We view the problem from the point of view of competitive analysis and thus assume that the algorithms have no knowledge of future arrivals of agents, as in the case of all existing online MAPF algorithms \cite{ho2019multi,vsvancara2019online}, although such knowledge might be learned in practice.

As our first contribution, we formalize online MAPF as an extension of (offline) MAPF and demonstrate how to generalize existing NP-hardness and inapproximability results for MAPF to online MAPF. 

As our second contribution, we classify online MAPF algorithms based on (1) different controllability assumptions, namely at what time the system can plan paths for which sets of agents, into three categories: PLAN-NEW-SINGLE that plans only a path for one newly-revealed agent at a time, PLAN-NEW that plans paths only for newly-revealed agents, and PLAN-ALL that plans paths for all known agents and thus allows rerouting and (2) different rationality, namely how effective the planned paths are: optimally-rational algorithms that plan optimal paths for the given set of agents and rational algorithms (which are, in our opinion, the only algorithms worth considering, assuming no knowledge of future arrivals of agents) that plan paths at least asymptotically as effective as the naive baseline algorithm SEQUENCE that routes newly-revealed agents one at a time in sequence. These classifications cover all existing online MAPF algorithms in \citet{vsvancara2019online} and different settings, for example, where rerouting of robots is always allowed \cite{MaAAMAS17} or disallowed \cite{ho2019multi}, in real-world systems.
The relationships between these algorithms are summarized in Figure \ref{fig:relationships}.

As our third contribution, we study online MAPF algorithms under the competitive analysis framework. Specifically, we demonstrate how an arbitrary online MAPF algorithm can be rationalized and show that the competitive ratios of all rational online MAPF algorithms with respect to flowtime and makespan are both bounded from above by $\mathcal O(m)$ for an input sequence of $m$ agents. We then show that (1) the bounds are tight for all rational algorithms in PLAN-NEW-SINGLE and PLAN-NEW, (2) the competitive ratio is at least 4/3 with respect to flowtime and 3/2 with respect to makespan for all rational algorithms in PLAN-ALL, and (3) the competitive ratio is infinite with respect to latency for all rational algorithms. The results hold even for optimally-rational algorithms and on 4-neighbor 2D grids. Therefore, for the first time, we provide theoretical insights into the effectiveness of using MAPF algorithms in an online setting \cite{salzman2020research} and address some of the long-standing open questions such as whether planning for multiple agents is more effective than planning for only one agent at a time in an online setting, whether algorithms that allow rerouting are more effective than those that disallow, and whether acting optimally rationally can improve the effectiveness. The results are summarized in Table \ref{tab:summary}.

\begin{figure}\centering
	\includegraphics[width=.53\columnwidth]{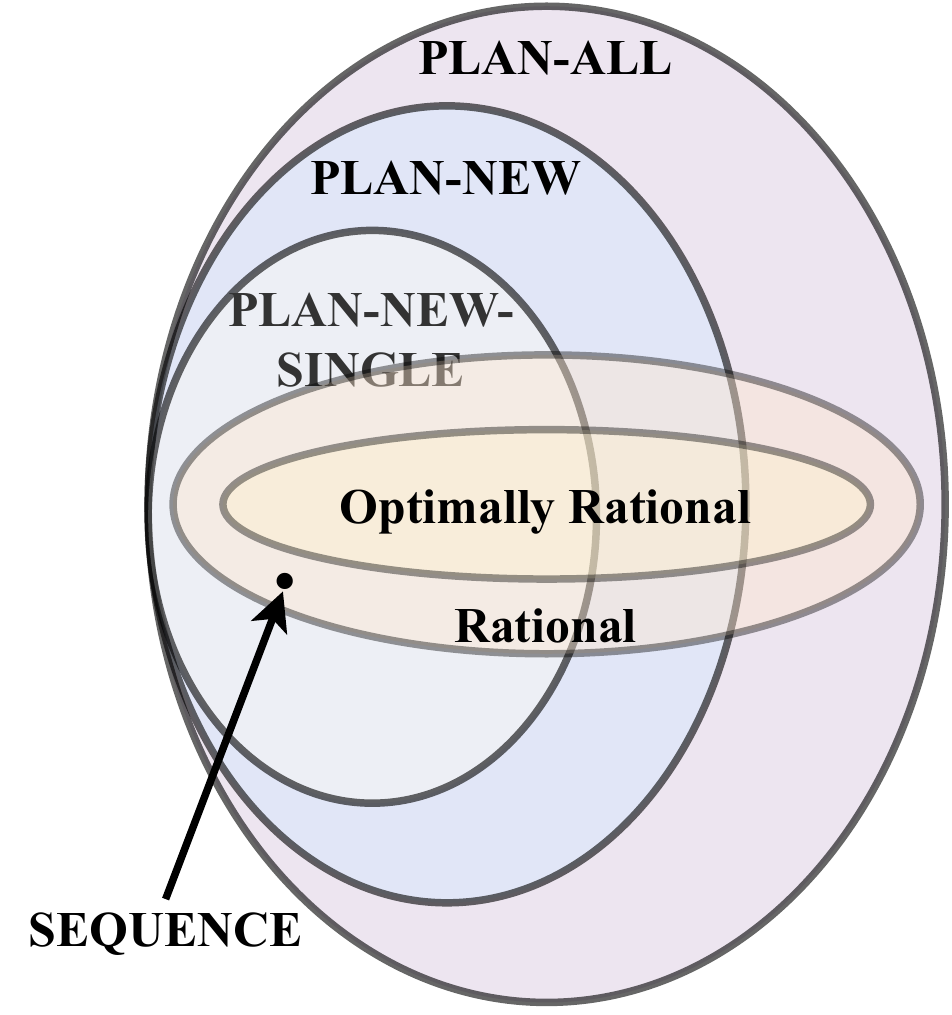}
	\caption{Relationships between online MAPF algorithms.}
	\label{fig:relationships}
\end{figure}

\begin{table}[t]
	\centering
	\Huge
	\resizebox{\columnwidth}{!}{%
%		\scriptsize
%		\begin{tabular}{@{\hskip0pt}c@{\hskip0pt}|@{\hskip0pt}c@{\hskip0pt}||c|c|c|c|c|c|c@{\hskip0pt}}
		\begin{tabular}{c|c||c|c|c|c|c|c|c}
			\hline
			\multicolumn{2}{c||}{Controllability}                                                                                               & \multicolumn{3}{c|}{PLAN-NEW-SINGLE}                                             & \multicolumn{2}{c|}{PLAN-NEW}                                         & \multicolumn{2}{c@{\hskip3pt}}{PLAN-ALL}                                         \\ \cline{1-9}
			\rot{\begin{tabular}[c]{@{}c@{}}Objective\\ Function\end{tabular}} & \rot{\begin{tabular}[c]{@{}c@{}}Competitive\\ Ratio\\ Bounds\end{tabular}} & \rot{SEQUENCE} & \rot{Rational} & \rot{\begin{tabular}[c]{@{}c@{}}Optimally\\ Rational\end{tabular}} & \rot{Rational} & \rot{\begin{tabular}[c]{@{}c@{}}Optimally\\ Rational\end{tabular}} & \rot{Rational} & \rot{\begin{tabular}[c]{@{}c@{}}Optimally\\ Rational\end{tabular}} \\ \hline\hline
			\multirow{3}{*}{\rot{flowtime}}                                    & upper                                                                & $\mathcal O(m)$ [Thm. \ref{thm:sequence_ub_flowtime}]& \multicolumn{6}{c}{$\mathcal O(m)$ [Thm. \ref{thm:flowtime_ub}]}                                                                                                                                                                                           \\ \cline{2-9}
			& \begin{tabular}[c]{@{}c@{}} lower (even \\on 2D grids)\end{tabular}                                                              & \multicolumn{5}{c|}{$\varOmega(m)$ [Thm. \ref{thm:new_flowtime_lb}]}    &  \multicolumn{2}{c}{4/3 [Thm. \ref{thm:all_flowtime_lb}]}                                      \\ \hline
			\multirow{3}{*}{\rot{makespan}}                                    & upper                                                                & $\mathcal O(m)$ [Thm. \ref{thm:sequence_ub_makespan})& \multicolumn{6}{c}{$\mathcal O(m)$ [Thm. \ref{thm:makespan_ub}]}                                                                                                                                                                                                 \\ \cline{2-9}
			& \begin{tabular}[c]{@{}c@{}} lower (even \\on 2D grids)\end{tabular}                                                                  & \multicolumn{5}{c|}{$\varOmega(m)$ [Thm. \ref{thm:new_makespan_lb}]}& \multicolumn{2}{c@{\hskip3pt}}{3/2 [Thm. \ref{thm:all_makespan_lb}]}                                      \\ \hline
			\rot{latency}                                                      
			&  \begin{tabular}[c]{@{}c@{}} (even on\\ 2D grids)\end{tabular}                                                                    & \multicolumn{7}{c@{\hskip3pt}}{$\infty$ [Obs. \ref{obs:latency}]}                                                                                                                                                                                             \\ \hline
		\end{tabular}
	}
	\caption{Summary of main competitiveness results.}\label{tab:summary}
\end{table}

%not understood: 
%
%hint on why snap-shot better?
%
%MAPF: start, goal pairwise different
%
%Online can be the same.
%
%Prove snap-shot optimal is in rational
%
%Draw relation chart
%
%	\cite{MaAAMAS17}
	
\section{Online Multi-Agent Path Finding}\label{sec:definition}

In an online MAPF instance, we are given a connected undirected graph $G = (V, E)$, whose vertices $V$ represent locations and whose edges $E$ represent connections between locations that the agents can traverse. We consider a finite input sequence of $m$ agents $a_1, a_2, \ldots, a_m$.
%although many of our results generalize to infinitely many agents.
Each agent $a_i$ is characterized by a \textit{start vertex} $s_i\in V$, a \textit{goal vertex} $g_i\in V$ that is different from the start vertex, and a non-negative integer \textit{release time} $r_i$, at which the agent is \textit{revealed} (appears and is added to the system) and available for routing. Release times are not known to online MAPF algorithms. Without loss of generality, we assume that the agents are given in non-decreasing order of its release time, that is, $r_1 \leq r_2 \leq \ldots \leq r_m$. Each agent $a_i$ can choose to start at any integer \textit{starting time} $t^{(s)}_i \ge r_i$, where it is added to graph $G$ at its start vertex. At each discrete time step, each agent $a_i$ either moves to an adjacent vertex or waits at the same vertex when it is in graph $G$.
%After the agent arrives at its goal vertex at \textit{arrival time} $t^{(g)}_i$, it is removed from graph $G$ at the next time step $t^{(g)}_i + 1$.
When the agent arrives at its goal vertex at \textit{arrival time} $t^{(g)}_i$, it is removed from graph $G$ (upon its arrival at time step $t^{(g)}_i$).

Let $\pi_i(t)\in V$ denote the vertex of agent $a_i$ at time step $t$. A \emph{path} $\pi_i = \langle \pi_i(t^{(s)}_i), \pi_i(t^{(s)}_i+1), \ldots, \pi_i(t^{(g)}) \rangle$ for agent $a_i$ satisfies the following condition: (1) The agent starts at its start vertex at the starting time $t^{(s)}_i$, that is, $\pi_i(t^{(s)}_i) = s_i$. (2) The agent ends at its goal vertex at the arrival time $t^{(g)}_i$, that is, $\pi_i(t^{(g)}_i) = g_i$ (but is removed from the graph upon its arrival and does not collide with other agents at time step $t^{(g)}_i$). (3) The agent always either moves to an adjacent vertex or waits at the same vertex between two consecutive time steps when in graph $G$, that is, for all time steps $t = t^{(s)},\ldots,t^{(g)}_i$, $(\pi_i(t), \pi_i(t+1)) \in E$ or $\pi_i(t+1) = \pi_i(t)$. Path $\pi$ is implicitly augmented with a null vertex $\bot\notin V$ for all time steps $t\in [0, \infty)\setminus[t^{(s)}_i,t^{(g)}_i]$.
Every two agents should avoid collisions with each other (only) when they are both in graph $G$: A \emph{vertex collision} occurs
%is a tuple $\langle a_i,a_j,v,t \rangle$, where
if two agents occupy the same vertex (except for their goal vertices)
%$v = \pi_i(t) = \pi_j(t) \in V$
at time step $t$. An \emph{edge collision} occurs
%is a tuple $\langle a_i,a_j,u,v,t \rangle$, where 
if two agents traverse the same edge 
%$(u,v) \in E$, where $u = \pi_i(t) = \pi_j (t+1)$ and $v = \pi_j (t) = \pi_i(t+1)$,
in opposite directions between time steps $t$ and $t+1$.

A \textit{plan} for a set $\mathcal A$ of agents consists of a path $\pi_i$ assigned to each agent $a_i\in \mathcal A$. Let $\mathit{dist}_i$ denote the length of the shortest path (optimal path cost) from vertex $s_i$ to vertex $g_i$ in graph $G$. The \textit{service time} $t^{(g)}_i - r_i$ of agent $a_i$ is the number of time steps for the agent to arrive at its goal vertex since it is revealed.
We consider three common objective functions:
\begin{enumerate}
	%[(1)]
	\item The \textit{flowtime} $\sum_{a_i\in\mathcal A} (t^{(g)}_i - r_i)$ is the sum of the service times $t^{(g)}_i - r_i$ of all agents in $\mathcal A$.
	\item The \textit{makespan} $\max_{a_i\in\mathcal A} t^{(g)}_i$ is the maximum of the arrival times of all agents in $\mathcal A$.\footnote{\citet{vsvancara2019online} claims that the makespan measure is problematic (probably because an infinite sequence of agents was considered there) and does not consider it. However, it is regarded as an important measure in the literature \cite{surynek2010optimization,YuLav13ICRA}, which corresponds to the earliest time when all (finitely many) transportation requests are served in practice.}
	\item The \textit{latency} $\sum_{a_i\in\mathcal A} (t^{(g)}_i - r_i - \mathit{dist}_i)$ is the sum of the differences between the service times $t^{(g)}_i - r_i$ and the distances $\mathit{dist}_i$ of all agents in $\mathcal A$.\footnote{The latency measure has been used in the multi-vehicle transportation research, for example, in the context of online ride and delivery services \cite{das2018minimizing}, and was first proposed by \citet{vsvancara2019online} for online MAPF.}
\end{enumerate}
Trivially, an online MAPF plan minimizes the flowtime if and only if it minimizes the latency, which can be rewritten as $\sum_{a_i\in\mathcal A} (t^{(g)}_i - r_i) - \sum_{a_i\in\mathcal A} \mathit{dist}_i$.
%It thus suffices to discuss only the first two objective functions.

An online MAPF \textit{solution} is a plan for all agents $a_1, \ldots, a_m$ whose paths are collision-free.

\subsection{Computational Complexity}\label{sec:complexity}

We show that online MAPF is NP-hard to solve (bounded-sub)optimally in general even with offline algorithms that know all agents a priori. Similar to  \citet{MaAAAI16,MaAAMAS18}, we use a reduction from an NP-complete version of the Boolean satifiability problem, called $\le$3,$=$3-SAT \cite{cat1984}.
A $\le$3,$=$3-SAT instance consists of $N$ Boolean variables and $M$ disjunctive clauses where each variable appears in exactly three clauses, uncomplemented at least once and complemented at least once, and each clause contains at most three literals. Its decision question asks whether there exists a satisfying assignment.

\begin{thm}\label{thm:makespan_hardness}
	For any $\epsilon > 0$, it is NP-hard to find a $4/3 - \epsilon$-approximate solution to online MAPF for makespan minimization, even if all agents are known a priori.
\end{thm}

%\begin{proof}[Proof Sketch]
%	Given a $\le$3,$=$3-SAT instance, we use a polynomial-time reduction similar to the one in \citet{MaAAAI16} to construct an online MAPF instance with $m=M+2N$ agents, all with release time 0, that has a solution with makespan three if and only if the $\le$3,$=$3-SAT instance is satisfiable and always has a solution with makespan four, even if the $\le$3,$=$3-SAT instance is unsatisfiable. For any $\epsilon>0$, any (offline or online) algorithm for online MAPF with approximation ratio $4/3 - \epsilon$ thus computes a solution with makespan three whenever the $\le$3,$=$3-SAT instance is satisfiable and thus solves $\le$3,$=$3-SAT.
%\end{proof}

The complete proof is given in the appendix.
In the proof of Theorem \ref{thm:makespan_hardness}, the online MAPF instance reduced from the given $\le$3,$=$3-SAT instance has the property that the
length of every path from the start vertex to the goal vertex of every agent is at least three. Therefore, if the makespan is three, then every agent arrives at its goal vertex in exactly three time steps and the flowtime is $3m$. Moreover, if the makespan exceeds three, then the flowtime exceeds $3m$, yielding the following corollary.

\begin{cor} \label{cor:flowtime_hardness}
	It is NP-hard to find an optimal solution to online MAPF for flowtime minimization, even if all agents are known a priori.
\end{cor}

In the proof of Theorem \ref{thm:makespan_hardness}, the constructed online MAPF instance has a solution with zero latency if and only if the given $\le$3,$=$3-SAT instance is satisfiable. Any algorithm for online MAPF with any constant approximation ratio $c$ thus computes a solution with zero latency whenever the given $\le$3,$=$3-SAT instance is satisfiable and thus solves $\le$3,$=$3-SAT, yielding the following corollary.

\begin{cor} \label{cor:latency_hardness}
	For any $c>0$, it is NP-hard to find a $c$-approximate solution to online MAPF for latency minimization, even if all agents are known a priori.
\end{cor}

\section{Online MAPF Algorithms}\label{sec:algorithm}

Online MAPF algorithms use the following assumption: For any given MAPF instance, the agents are partitioned into $K$ disjoint sets $\mathcal A_1, \ldots, \mathcal A_K$ based on their release times (a total of $K$ different values), where agents in each set $\mathcal A_k$ have the same release time $r_{\mathcal A_k}$ for all $k\in [K]$ (where $[K] = \{1, \ldots, K\}$) and the sets $\mathcal A_k$ are indexed in increasing order of $r_{\mathcal A_k}$. Let $\mathcal A_{\leq k} = \bigcup_{k'\in [k]}\mathcal A_{k'}$ denote the set of all revealed agents by release time $r_{\mathcal A_k}$, for all $k\in [K]$, and $\mathcal A_{< k} = \bigcup_{k' \in [k-1]}\mathcal A_{k'}$ the set of all previously-revealed agents at each release time $r_{\mathcal A_k}$, for all $k\in [2, K]$.
At each
%release time
$r_{\mathcal A_k}$, an online MAPF algorithm calls an offline path-finding algorithm to plan paths so that all revealed agents have paths.

\subsection{Controllability}
We categorize online MAPF algorithms based on different \textit{controllability assumptions} about for which \textit{controllable set} $\mathcal A^\mathcal C$ of agents they can plan paths at each time when they call an offline path-finding algorithm. Algorithm \ref{alg:onlineMAPF} shows the pseudo-code of a template of online MAPF algorithms. At each release time $r_{\mathcal A_k}$, an online MAPF algorithm considers each controllable set of agents [Line \ref{line:each_set}]. It then treats all agents that have an already planned path and are not in the controllable set as \textit{dynamic obstacles} that follow their already planned paths [Line \ref{line:obstacles}] and computes a plan for all agents in the controllable set [Line \ref{line:plan}]. We highlight the following three controllability categories.
\begin{enumerate}
	%[(1)]
	\item\relax[$\mathcal A^\mathcal C = \{a_i\}$, $\forall a_i \in \mathcal A_k$] Algorithms in \textbf{PLAN-NEW-SINGLE} call an offline single-agent path-finding algorithm to plan a path for one (newly-revealed) agent $a_i$ in $\mathcal A_k$ at a time in increasing order of the indices of the agents and treat all agents $a_j$ with
	%indices smaller than it, namely
	$j<i$ as dynamic obstacles, until all agents in $\mathcal A_k$ have paths.
	\item\relax[$\mathcal A^\mathcal C = \mathcal A_k$] Algorithms in \textbf{PLAN-NEW} call an offline MAPF algorithm to plan paths for all newly-revealed agents, namely agents in $\mathcal A_k$, and treat all previously-revealed agents as dynamic obstacles.
	\item\relax[$\mathcal A^\mathcal C = \mathcal A_{\leq k}$] Algorithms in \textbf{PLAN-ALL} call an offline MAPF algorithm to plan paths for all revealed agents from time step $r_{\mathcal A_k}$ on, thus allowing previously-revealed agents to change their paths from time step $r_{\mathcal A_k}$ on.
\end{enumerate}
Therefore, online MAPF algorithms in the same category differ from each other only in the offline single/multi-agent path-finding algorithm they used to solve each single/multi-agent path-finding problem. 
Trivially,
%$\text{PLAN-NEW-SINGLE}\subseteq\text{PLAN-NEW}\subseteq\text{PLAN-ALL}$, namely
PLAN-NEW-SINGLE is a subset of PLAN-NEW, which, in turn, is a subset of PLAN-ALL, since any online MAPF algorithm in PLAN-NEW-SINGLE can be viewed as a special case of PLAN-NEW that plans a path for each newly-revealed agent in sequence and any online MAPF algorithm in PLAN-NEW can be viewed as a special case of PLAN-ALL that always plans the same paths as the already planned ones for all previously-revealed agents.

In practice, different controllability assumptions correspond to different types of real-world systems: Algorithms in PLAN-NEW-SINGLE can easily be adapted to most real-time distributed systems since agents make decisions individually and require less communication with each other. PLAN-NEW corresponds to centralized systems where agents cannot be rerouted easily, for example, automated parcel sortation centers with agents moving at high speeds \cite{kou2020idle}. PLAN-ALL corresponds to centralized systems that allow frequent rerouting, for example, a recent proposal of automated train (re-)scheduling systems by the Swiss Federal Railways \cite{flatland2020,LiICAPS21}.
Controllability often depends on specific applications and might also be viewed as part of the problem definition.

\begin{algorithm}[t]	
	\caption{Online MAPF Algorithm Template}\label{alg:onlineMAPF}
	\KwIn{online MAPF instance}
	\tcc{system executes at release time $r_{\mathcal A_k}$}
	\ForEach{controllable set $\mathcal A^\mathcal C$\label{line:each_set}}
	{
		$\bar{\mathcal A} \gets \{a_i|a_i\text{ has a path \textbf{and} } a_i\notin A^\mathcal C\}$\label{line:obstacles}\;
		Compute a plan for $\mathcal A^\mathcal C$ that treats $\bar{\mathcal A}$ as dynamic obstacles\label{line:plan}\;
	}
	\tcc{system advances to the next release time}
\end{algorithm}

%We also define a category \textbf{OFFLINE} that consists of all offline algorithms, which have full controllability of agents $a_1, \ldots, a_m$ at time step 0.

\subsection{Optimal Rationality}

We now discuss one category of online MAPF algorithms based on the quality of paths the algorithms plan for the controllable set of agents at each time. This categorization is orthogonal to the controllability assumptions and more algorithmic than application-specific.

Recent research \cite{vsvancara2019online} has shown that ``snapshot-optimal'' online MAPF algorithms in PLAN-ALL that (use an optimal offline MAPF algorithm to) compute a plan for all revealed agents with the smallest flowtime at each release time tend to result in an online MAPF solution with small flowtime. We generalize this notion of optimality to algorithms in all the above three controllability categories.
\begin{dfn}[Optimal Rationality]
	A plan for a given controllable set $\mathcal A^\mathcal C$ of agents at a given release time is \textit{optimally-rational} if and only if it results in a plan for all planned agents (agents that have paths) with the smallest cost with respect to a given objective function (under a given controllability assumption).
	An online MAPF algorithm is \textit{optimally-rational} if and only if it computes an optimally-rational plan for the controllable set of agents at each time when it calls an offline path-finding algorithm.
\end{dfn}

Trivially, an online MAPF algorithm is optimally-rational with respect to flowtime if and only if it is optimally-rational with respect to latency. We now give examples of offline path-finding algorithms that can be used in optimally-rational online MAPF algorithms under different controllability assumptions with respect to flowtime and makespan.

\noindent\textbf{Examples of Optimally-Rational Algorithms:}
Optimally-rational algorithms in PLAN-NEW-SINGLE with respect to flowtime and makespan can both call Space-Time A* \cite{WHCA} or Safe Interval Path Planning (SIPP) \cite{SIPP} at release time $r_{\mathcal A_k}$ to find a path for each agent $a_i\in \mathcal A_k$ with the smallest arrival time $t^{(g)}_i$, which are similar in essence to prioritized offline MAPF algorithms (for example, Cooperative A*\cite{WHCA}). 
%Alternatively, optimally-rational algorithms in PLAN-NEW-SINGLE with respect to makespan can also call a modified version of Space-Time A* or SIPP at release time $r_{\mathcal A_k}$ to find a path for each agent $a_i\in \mathcal A_k$ with the smallest cost with respect to a custom cascading cost function where any path with arrival time $t^{(g)}_i$ no larger than the makespan of the plan for agents whose paths have been planned, namely $\{a_j|j<i\}$, has cost equal to the makespan and, otherwise, has cost equal to the arrival time $t^{(g)}_i$, which still results in a plan for all agents whose paths have been planned with the smallest makespan.
Optimally-rational algorithms in PLAN-NEW (respectively, PLAN-ALL) with respect to flowtime can call any optimal offline MAPF algorithm at release time $r_{\mathcal A_k}$ to find a plan for all agents in $\mathcal A_k$ (respectively, $\mathcal A_{\leq k}$) with the smallest flowtime. Optimally-rational algorithms in PLAN-NEW (respectively, PLAN-ALL) with respect to makespan can call any optimal offline MAPF algorithm at release time $r_{\mathcal A_k}$ to find a plan for all agents in $\mathcal A_k$ (respectively, $\mathcal A_{\leq k}$) with the smallest makespan. Note that, unlike the above examples,
optimally-rational algorithms with respect to makespan can alternatively find a plan for agents in $\mathcal A^\mathcal C$ that does not necessarily have the smallest makespan but still results in a plan for all planned agents with the smallest makespan (and is thus optimally-rational).
%Alternatively, like PLAN-NEW-SINGLE, optimally-rational algorithms in PLAN-NEW with respect to makespan can also call a modified offline MAPF algorithm at release time $r_{\mathcal A_k}$ to find a plan for all agents in $\mathcal A_k$, not necessarily with the smallest makespan, that results in a plan for all agents in $\mathcal A_{\leq k}$ with the smallest makespan.

\section{Feasibility and SEQUENCE}

We now show that, unlike (offline) MAPF (in which agents are not removed
%from the graph
upon arrival at goal vertices), all online MAPF instances are solvable.
We prove the statement by describing the following naive online MAPF algorithm \textbf{SEQUENCE} that routes agents one after another in sequence.

SEQUENCE plans a path for one agent at a time in increasing order of the indices of the agents (thus also non-decreasing order of the release times of the agents). Specifically, it first plans a path $\pi_1$ for agent $a_1$ at time step $r_1$, when the agent is revealed, such that the agent starts from vertex $s_1$ at time step $t^{(s)}_1 = r_1$, moves along the shortest path from vertex $s_1$ to vertex $g_1$ in graph $G$, and arrives at vertex $g_1$ at time step $t^{(g)}_1 = t^{(s)}_1 + \mathit{dist}_1$. Then, for each $i=2,\ldots,m$, it plans a path $\pi_i$ for agent $a_i$ at time step $r_i$ such that the agent starts at time step $t^{(s)}_{i} = \max(r_i, t^{(g)}_{i-1})$, moves along the shortest path from vertex $s_i$ to vertex $g_i$ in graph $G$, and arrives at vertex $g_i$ at time step $t^{(g)}_i = t^{(s)}_i + \mathit{dist}_i$. In other words, each such agent starts routing only when the previous agent has finished routing and been removed.
% from the graph.

\begin{obs}
	All online MAPF instances are solvable, and SEQUENCE solves them.
\end{obs}

\begin{proof}[Reason]
	The shortest path computation of paths $\pi_i$ for all $i\in[m]$ succeeds because graph $G$ is connected. Each agent thus arrives at its goal vertex at a finite time step. Also, the resulting paths are collision-free since no two agents are in graph $G$ at the same time step.
\end{proof}

\subsection{Competitive Ratio Upper Bounds}

We follow the standard definition of competitive ratio \cite{borodin2005online}.
Let the cost of online algorithm \textit{ALG} for an input sequence $\sigma$ (in our case, the sequence of agents%
%of a given online MAPF instance
) be $C_\textit{ALG}(\sigma)$ with respective to a given objective function and the cost of an optimal offline algorithm \textit{OPT} that knows the entire input sequence $\sigma$ a priori (in our case, has full controllability at time step 0 of all agents $a_1, \ldots, a_m$ that might be revealed in the future) be $C_\textit{OPT}(\sigma)$.

\begin{dfn}[Competitive Ratio]
	An online algorithm \textit{ALG} is \textit{$\alpha$-competitive} or has a \textit{competitive ratio} of $\alpha$ if, for all input sequence $\sigma$ and some constant $\delta$,  $C_\textit{ALG}(\sigma) \leq \alpha C_\textit{OPT}(\sigma) + \delta$.
\end{dfn}

We now derive upper bounds on the competitive ratio for SEQUENCE with respect to flowtime and makespan in the following theorems, respectively.

\begin{thm}\label{thm:sequence_ub_flowtime}
	SEQUENCE achieves a competitive ratio of $\mathcal O(m)$ with respect to flowtime.
\end{thm}

\begin{proof}
	We first show by induction on $i$ that the service time $t^{(g)}_i - r_i$ of each agent $a_i$ is no larger than $\sum_{j\in[i]}\mathit{dist}_j$. The statement holds trivially for agent $a_1$. Assume that its holds for agent $a_{i-1}$. 
	The service time of agent $a_i$ is thus
	{\small\begin{align}
	t^{(g)}_i - r_i  &=&& t^{(s)}_i + \mathit{dist}_i - r_i\nonumber\\[-.25ex]
	&=&& \max(r_i, t^{(g)}_{i-1}) + \mathit{dist}_i - r_i\nonumber\\[-.25ex]
	&=&& \max(0, t^{(g)}_{i-1} - r_i) + \mathit{dist}_i\nonumber\\[-.25ex]
	&\stackrel{\text{definition}}{\leq}&& \max(0, t^{(g)}_{i-1} - r_{i-1}) + \mathit{dist}_i\nonumber\\[-.25ex]
	&=&& t^{(g)}_{i-1} - r_{i-1} + \mathit{dist}_i\nonumber\\[-.25ex]
	&\stackrel{\text{induction}}{\leq}&& \sum_{j\in[i-1]}\mathit{dist}_j + \mathit{dist}_i = \sum_{j\in[i]}\mathit{dist}_j.\label{eqn:flow}
	\end{align}}%
	The statement thus holds also for agent $a_i$. The flowtime of the plan for all agents is thus no larger than $\sum_{i\in[m]}\sum_{j\in[i]}\mathit{dist}_j \leq m\sum_{j\in[m]}\mathit{dist}_j$. Since the optimal flowtime is no smaller than $\sum_{j\in[m]}\mathit{dist}_j$, the theorem follows.
\end{proof}

%We then derive an upper bound with respect to makespan using an argument similar to the one used in the above proof.
%We let $\rho_{n}$ be the latest release time $r_i$ with $i\leq n$ such that $r_i > r_1 + \sum_{j\in[i-1]}\mathit{dist}_j$ or $\rho_{n}=r_1$ if there is no such release time.

\begin{thm}\label{thm:sequence_ub_makespan}
	SEQUENCE achieves a competitive ratio of $\mathcal O(m)$ with respect to makespan.
\end{thm}

\begin{proof}
	Let $r_{n_K}$ be the latest release time with $r_{n_K} > r_1 + \sum_{j\in[n_K-1]}\mathit{dist}_j$ or $r_{n_K}=r_1$ if there is no such release time. According to Equation (\ref{eqn:flow}), SEQUENCE computes a solution with makespan no larger than $r_{n_K} + \sum_{j\in[n_K,m]}\mathit{dist}_j\leq r_{n_K} + m\max_{j\in[n_K,m]}\mathit{dist}_j$.
%	{\small\begin{align*}
%	&\leq&&r_n + \sum_{i\in[n,m]}\mathit{dist}_i\\[-.25ex]
%	&\leq&& r_n + m\max_{i\in[n,m]}\mathit{dist}_i.
%	\end{align*}}%
	Since the optimal makespan is no smaller than $\max_{j\in[n_K,m]}(r_{n_K} + \mathit{dist}_j)$, the theorem follows.
\end{proof}

We show in Section \ref{sec:infinite_latency} that the competitive ratio for SEQUENCE is infinite with respect to latency.

\section{Rationality and Competitive Ratio Upper Bounds}\label{sec:upperbound}

The $\mathcal O(m)$ upper bounds on the competitive ratio for the naive algorithm SEQUENCE shown in Theorems \ref{thm:sequence_ub_flowtime} and \ref{thm:sequence_ub_makespan} set a baseline for all online MAPF algorithms about what behavior (quality of the plans returned by the algorithms) is rational. This analysis also inspires the characterization of \textit{rational} algorithms that are guaranteed to result in a solution quality (asymptotically) no worse than SEQUENCE, which extends the notion of optimally-rational algorithms. Note that our definition of rationality is significantly different from that in the optimization and economics literature.
%\footnote{Our definition of rationality is significantly different from that in the optimization and economics literature.}
%It is thus only meaningful to consider only
\begin{dfn}[Rationality]\label{dfn:rationality}
	An online MAPF algorithm is \textit{rational} if and only if, at each release time $r_{\mathcal A_k}$, $\forall k\in [K]$, the plan for
	all revealed agents	(agents in $\mathcal A_{\leq k}$)
	has flowtime no larger than $|\mathcal A_{\leq k}|\sum_{a_i\in \mathcal A_{\leq k}}\mathit{dist}_i$ and makespan no larger than $r_{n_k} + \sum_{i\in[n_k,m_k]}\mathit{dist}_i$ where $m_k=\argmax_i a_i \in \mathcal A_{\leq k}$ and $r_{n_k}\leq r_{\mathcal A_k}$ is the latest release time of agent $a_{n_k}$ with $r_{n_k} > r_1 + \sum_{i\in[n_k-1]}\mathit{dist}_i$ or $r_{n_k}=r_1$ if there is no such release time.
\end{dfn}
The main idea behind Definition \ref{dfn:rationality} is to set upper bounds that quantify how badly an online MAPF algorithm could perform at each release time $r_{\mathcal A_k}$ with respect to both flowtime and makespan so that it still results in a solution quality that is asymptotically no worse than SEQUENCE. Trivially, SEQUENCE is rational by using a similar argument as in the proofs of Theorems \ref{thm:sequence_ub_flowtime} and \ref{thm:sequence_ub_makespan} for each release time $r_{\mathcal A_k}$.
%according to Definition \ref{dfn:rationality}. 
The flowtime term in Definition \ref{dfn:rationality} can be rewritten with respect to latency. However, we show in Section \ref{sec:infinite_latency} that the competitive ratio for all rational algorithms is (at least) infinite with respect to latency and thus do not study its upper bound in this section.

\subsection{Rationalization}
Not all online MAPF algorithms are rational. However, one can \textit{rationalize} any online MAPF algorithm in PLAN-NEW and PLAN-ALL by adding a simple subroutine to it as follows: 
At each release time $r_{\mathcal A_k}$, $\forall k\in [K]$, the online MAPF algorithm calls the subroutine at the end of its computation that checks whether the resulting plan for all agents in $\mathcal A_k$ respects the upper bounds set by Definition \ref{dfn:rationality}. If so, it changes the plan so that the agents move to their goal vertices one after another in increasing order of their indices, as in SEQUENCE, starting at time step $t_{\mathcal A_k}^{(s)}$ that is the maximum of $r_{\mathcal A_k}$ and the makespan of the (old) computed plan at the previous release time $r_{\mathcal A_{k-1}}$ for $k > 1$ and time step $t_{\mathcal A_k}^{(s)}=r_1$ for $k = 1$.
The key idea
%of rationalization for an online MAPF algorithm
is that, after each computation, it checks whether the resulting plan (for the set of agents that have paths) respects the upper bounds set by Definition \ref{dfn:rationality} and, if not, switches to the same behavior as that of SEQUENCE to guarantee that the plan is asymptotically no worse than the plan (for the same set of agents) in SEQUENCE.
We have omitted the discussion of rationalization for PLAN-NEW-SINGLE above, which can be achieved by adding a similar subroutine after each single-agent path-finding computation that compares the flowtime and makespan of the resulting plan to those in SEQUENCE.

\subsection{Rationality of Optimally-Rational Algorithms}

Intuitively, online MAPF algorithms that look reasonable, for example, ones that are optimally-rational, are rational (without rationalization). We prove the following theorem, even though the name ``optimal rationality'' has already suggested it.

\begin{thm}\label{thm:snapshot_rationality}
	Optimally-rational online MAPF algorithms under any given controllability assumption with respect to either flowtime or makespan are rational.
\end{thm}

\begin{proof}
	We consider an arbitrary optimally-rational online MAPF algorithm in PLAN-NEW-SINGLE with respect to either flowtime or makespan. We show that, at each release time $r_{\mathcal A_k}$, $\forall k\in [K]$, the plan for all revealed agents has flowtime no larger than $|\mathcal A_{\leq k}|\sum_{a_i\in \mathcal A_{\leq k}}\mathit{dist}_i$ and makespan no larger than $r_{n_k} + \sum_{i\in[n_k,m_k]}\mathit{dist}_i$ and the algorithm is thus rational (Definition \ref{dfn:rationality}). We first show by induction on $i$ that the service time $t^{(g)}_i - r_i$ of each agent $a_i$ is no larger than $\sum_{j\in[i]}\mathit{dist}_j$. The statement holds trivially for agent $a_1$ for such an algorithm. Assume that it holds for agents $a_2, \ldots, a_{i-1}$. 
	The service time of agent $a_i$ is thus
	{\small\begin{align*}
			t^{(g)}_i - r_i &=&& t^{(s)}_i + \mathit{dist}_i - r_i\\[-.25ex]
			&\stackrel{\text{opt. rat.}}{\leq}&& \max(r_i, \max_{j\in[i-1]} t^{(g)}_j) + \mathit{dist}_i - r_i\\[-.25ex]
			&\leq&& \max(0, \max_{j\in[i-1]} (t^{(g)}_j - r_i)) + \mathit{dist}_i\\[-.25ex]
			%&=&& \max_{j\in[i-1]} (t^{(g)}_j - r_i) + \mathit{dist}_i\\[-.25ex]
			&\stackrel{\text{definition}}{\leq}&& \max(0, \max_{j\in[i-1]} (t^{(g)}_j - r_j)) + \mathit{dist}_i\\[-.25ex]
			&\stackrel{\text{induction}}{\leq}&&  \max(0, \max_{j\in[i-1]}\sum_{j'\in[j]}\mathit{dist}_{j'}) + \mathit{dist}_i\\[-.25ex]
			&=&&  \sum_{j\in[i-1]}\mathit{dist}_j + \mathit{dist}_i = \sum_{j\in[i]}\mathit{dist}_j
	\end{align*}}%
	for any optimally-rational online MAPF algorithm in PLAN-NEW-SINGLE with respect to either flowtime or makespan.
	%Therefore, using an argument similar to the one for deriving Equation (\ref{eqn:flow}) in the proof of Theorem \ref{thm:sequence_ub_flowtime},
	The statement thus holds also for agent $a_i$. The plan for all revealed agents at any release time $r_{\mathcal A_k}$ thus has flowtime $F_k$ no larger than $\sum_{i\in \mathcal A_{\leq k}}\sum_{j\in[i]}\mathit{dist}_j \leq |\mathcal A_{\leq k}|\sum_{j\in\mathcal A_{\leq k}}\mathit{dist}_j$ and makespan $M_k$ no larger than $r_{n_k} + \sum_{j\in[n_k,m_k]}\mathit{dist}_j$.
	For any optimally-rational online MAPF algorithm in PLAN-NEW or PLAN-ALL with respect to either flowtime or makespan, the plan for all revealed agents at each release time $r_{\mathcal A_k}$ has flowtime no larger than $F_k$ and makespan no larger than $M_k$ (or it is not optimally-rational otherwise). Rationality is thus maintained at each release time $r_{\mathcal A_k}$.
\end{proof}

\subsection{Competitive Ratio Upper Bounds}

We show that all rational online MAPF algorithms perform asymptotically no worse than SEQUENCE with respect to both flowtime and makespan.

\begin{thm}\label{thm:flowtime_ub}
	All rational online MAPF algorithms achieve a competitive ratio of $\mathcal O(m)$ with respect to flowtime.
\end{thm}

\begin{proof}
	According to Definition \ref{dfn:rationality}, the plan computed by any rational online MAPF algorithm at the latest release time $r_{\mathcal A_K} = r_m$ has flowtime no larger than 
	$|\mathcal A_{\leq K}|\sum_{a_i\in \mathcal A_{\leq K}}\mathit{dist}_i = m\sum_{i\in[m]}\mathit{dist}_i.$
	Since the optimal flowtime is no smaller than $\sum_{i\in[m]}\mathit{dist}_i$, the theorem follows.
\end{proof}

\begin{thm}\label{thm:makespan_ub}
	All rational online MAPF algorithms achieve a competitive ratio of $\mathcal O(m)$ with respect to makespan.
\end{thm}

\begin{proof}
	According to Definition \ref{dfn:rationality}, the plan computed by any rational online MAPF algorithm at the latest release time $r_{\mathcal A_K} = r_m$ has makespan no larger than $r_{n_K} + \sum_{i\in[n_K,m]}\mathit{dist}_i\leq r_{n_K} + m\max_{i\in[n_K,m]}\mathit{dist}_i$. Since the optimal makespan is no smaller than $\max_{i\in[n_K,m]}(r_{n_K} + \mathit{dist}_i)$, the theorem follows.
\end{proof}

\section{Competitive Ratio Lower Bounds}\label{sec:lowerbound}

We show that all rational online MAPF algorithms in PLAN-NEW achieve a competitive ratio of at least $\varOmega(m)$ with respect to both flowtime and makespan by constructing an online MAPF instance on a 4-neighbor 2D grid. Therefore, the $\mathcal O(m)$ upper bounds on the competitive ratio of rational online MAPF algorithms in PLAN-NEW that we derived in Theorems \ref{thm:flowtime_ub} and \ref{thm:makespan_ub} are (asymptotically) tight even for online MAPF instances on 4-neighbor 2D grids. Consequently, we show that rational algorithms do not necessarily outperform irrational ones. We also derive lower bounds on the competitive ratio of rational online MAPF algorithms in PLAN-ALL with respect to flowtime and makespan. Finally, we show that all rational online MAPF algorithms have infinite competitive ratio with respect to latency.

\subsection{Rational Algorithms in PLAN-NEW}

The following theorems show that the competitive ratio of all rational online MAPF algorithms in PLAN-NEW, including optimally-rational ones, is at least $\varOmega(m)$ even on 4-neighbor 2D grids (and even lines) with respect to both flowtime and makespan.

\begin{thm}\label{thm:new_flowtime_lb}
	There exists an online MAPF instance on a 4-neighbor 2D grid for which any rational online MAPF algorithm in PLAN-NEW achieves a competitive ratio of $\varOmega(m)$ with respect to flowtime.
\end{thm}

\begin{proof}
	Consider an online MAPF instance on the 4-neighbor 2D grid shown in Figure \ref{fig:new_flowtime_lb} where an even number $m$ of agents are given and agent $a_i$ has release time $r_i = i - 1$, for all $i\in[m]$. Agent $a_i$ has start vertex $s_i = v_0$ and goal vertex $g_i = v_m$ if $i$ is odd and start vertex $s_i = v_m$ and goal vertex $g_i = v_0$ if $i$ is even. Consider an arbitrary rational online MAPF algorithm in PLAN-NEW. At time step 0, the algorithm computes the only possible path for agent $a_1$ where it starts at time step 0, moves without waiting, and arrives at $g_1 = v_m$ at time step $m$ since other paths violate the makespan upper bound in Definition \ref{dfn:rationality}. At time step 1, the algorithm computes the only possible path for agent $a_2$ where it starts at time step $m$ when agent $a_1$ has arrived at $g_1$ (or it collides with agent $a_1$ otherwise), moves without waiting, and arrives at $g_2 = v_0$ at time step $2m$ since other paths violate the makespan upper bound in Definition \ref{dfn:rationality} ($r_{n_2} = r_1$). We apply the argument to each agent $a_i$, $i \in [3,m]$, that it has to start after agent $a_{i-1}$ has arrived at $g_{i-1}$ (since the paths for all previously-revealed agents do not change for algorithms in PLAN-NEW) and moves to its goal vertex without waiting at any vertex to maintain rationality (Definition  \ref{dfn:rationality} with $r_{n_k} = r_1$ for release time $r_{\mathcal A_k}, \forall k\in[K], K=m$). The resulting flowtime is $m + (2m - 1) + (3m - 2) + \ldots + (m^2 - m + 1) = \frac12m^3 + \frac12m$. The optimal plan lets agents $a_i$ start at their release times $r_i$ if $i$ is odd and at time steps $2m - 3 + \frac i2$ if $i$ is even (from time step $2m-2$ on when the last agent $a_{m-1}$ with an odd index has arrived at its goal vertex). All agents move to their goal vertices without waiting once they start, resulting in service time $m$ for all odd numbers $i$ and $3m - 2 - \frac i2$ for all even numbers $i$. The optimal flowtime is thus $\frac{15}8m^2 - \frac54m$. The theorem thus follows.
\end{proof}

\begin{figure}
\centering
\begin{subfigure}[b]{0.49\columnwidth}
	\centering
	\begin{tikzpicture}[scale=0.5]
	\draw[thick, step=1cm] (0,0) grid (3.2,1);
	\node at (.5,.5) {$v_0$};
	\node at (1.5,.5) {$v_1$};
	\node at (2.5,.5) {$v_2$};
	\node at (4,.5) {\huge ...};
	\draw[thick, step=1cm] (6,1) grid (4.8,0);
	\node at (5.5,.5) {$v_m$};		
	\end{tikzpicture}
	\caption{Theorem \ref{thm:new_flowtime_lb}.}
	\label{fig:new_flowtime_lb}
\end{subfigure}
\hfill
\begin{subfigure}[b]{0.49\columnwidth}
	\centering
	\begin{tikzpicture}[scale=0.5]
	\draw[thick, step=1cm] (-1,-1) grid (1,1);
	%	\node at (-0.25,0.25) {$s_1$};
	%	\node at (0.75,-0.75) {$g_1$};
	%	\node at (0.25,0.75) {$s_2$};
	%	\node at (-0.75,+0.75) {$g_2$};	
	\node at (-0.5,0.5) {$v_1$};
	\node at (0.5,0.5) {$v_2$};
	\node at (-0.5,-0.5) {$v_3$};
	\node at (0.5,-0.5) {$v_4$};	
	\end{tikzpicture}
	\caption{Theorem \ref{thm:all_flowtime_lb}.}\label{fig:all_flowtime_lb}
\end{subfigure}
\caption{Online MAPF instances used for theorems.}
\end{figure}
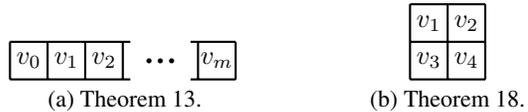

\begin{thm}\label{thm:new_makespan_lb}
	There exists an online MAPF instance on a 4-neighbor 2D grid for which any rational online MAPF algorithm in PLAN-NEW achieves a competitive ratio of $\varOmega(m)$ with respect to makespan.
\end{thm}

%\begin{proof}
%	We use the same online MAPF instance and argument as in the proof of Theorem \ref{thm:new_flowtime_lb}. Any rational online MAPF algorithm in PLAN-NEW results in makespan $m^2 + m - 1$. The optimal plan has makespan $\frac72m - 2$. The theorem thus follows.
%\end{proof}

\begin{proof}
	Using the same online MAPF instance and argument as in the proof of Theorem \ref{thm:new_flowtime_lb}, any rational online MAPF algorithm in PLAN-NEW results in makespan $m^2$. The optimal plan (with respect to both flowtime and makespan) has makespan $\frac72m - 3$. The theorem thus follows.
\end{proof}

Theorems \ref{thm:new_flowtime_lb} and \ref{thm:new_makespan_lb} thus show that the upper bounds that we derived in Theorems \ref{thm:flowtime_ub} and \ref{thm:makespan_ub} for flowtime and makespan, respectively, are asymptotically tight for all rational online MAPF algorithms in PLAN-NEW, yielding the following corollaries.

\begin{cor}
	All rational online MAPF algorithms in PLAN-NEW are $\varTheta(m)$-competitive with respect to flowtime even on 4-neighbor 2D grids.
\end{cor}

\begin{cor}
	All rational online MAPF algorithms in PLAN-NEW are $\varTheta(m)$-competitive with respect to makespan even on 4-neighbor 2D grids.
\end{cor}

\subsection{Irrational Algorithms}

A counter-intuitive insight we obtain from the above theorems is that rational online MAPF algorithms do not necessarily outperform the irrational ones and rationalization could harm the solution quality.

\begin{obs}
	 There exists an online MAPF instance on a 4-neighbor 2D grid for which an irrational online MAPF algorithm in PLAN-NEW-SINGLE outperforms any rational and thus rationalized online MAPF algorithms in PLAN-NEW with respect to both flowtime and makespan.
\end{obs}

\begin{proof}[Reason]
	For the online MAPF instance used in the proofs of Theorems \ref{thm:new_flowtime_lb} and \ref{thm:new_makespan_lb}, a dummy online MAPF algorithm that imitates the optimal offline algorithm, which can be viewed as in PLAN-NEW-SINGLE, is irrational and outperforms any rational online MAPF algorithm in PLAN-NEW.
\end{proof}

However, it is ``irrational'' and impossible for one to design such a perfect irrational online MAPF algorithm, for example, one that delays agent $a_2$ by $m$ time steps in the above example, that imitates the behavior of the optimal offline algorithm in practice without any knowledge of future arrival of agents. Also note that the above reasoning does not carry over to rational online MAPF algorithms in PLAN-ALL.

\subsection{Rational Algorithms in PLAN-ALL}

We now construct an online MAPF instance on a 4-neighbor 2D grid for which all rational algorithms in PLAN-ALL achieve a constant competitive ratio.

\begin{thm}\label{thm:all_flowtime_lb}
	There exists an online MAPF instance on a 4-neighbor 2D grid for which any rational online MAPF algorithm achieves a competitive ratio of at least 4/3 with respect to flowtime.
\end{thm}

%\begin{figure}\centering
%	\begin{tikzpicture}[scale=0.5]
%	\draw[thick, step=1cm] (-1,-1) grid (1,1);
%%	\node at (-0.25,0.25) {$s_1$};
%%	\node at (0.75,-0.75) {$g_1$};
%%	\node at (0.25,0.75) {$s_2$};
%%	\node at (-0.75,+0.75) {$g_2$};	
%	\node at (-0.5,0.5) {$v_1$};
%	\node at (0.5,0.5) {$v_2$};
%	\node at (-0.5,-0.5) {$v_3$};
%	\node at (0.5,-0.5) {$v_4$};	
%	\end{tikzpicture}
%	\caption{Online MAPF instance used for Theorem \ref{thm:all_flowtime_lb}.}\label{fig:all_flowtime_lb}
%\end{figure}

%\begin{figure}
%	\centering
%		\begin{tikzpicture}[scale=0.5]
%			\draw[thick, step=1cm] (-1,-1) grid (1,1);
%			%	\node at (-0.25,0.25) {$s_1$};
%			%	\node at (0.75,-0.75) {$g_1$};
%			%	\node at (0.25,0.75) {$s_2$};
%			%	\node at (-0.75,+0.75) {$g_2$};	
%			\node at (-0.5,0.5) {$v_1$};
%			\node at (0.5,0.5) {$v_2$};
%			\node at (-0.5,-0.5) {$v_3$};
%			\node at (0.5,-0.5) {$v_4$};	
%		\end{tikzpicture}
%	\caption{Online MAPF instance used for Theorem \ref{thm:all_flowtime_lb}.}\label{fig:all_flowtime_lb}
%\end{figure}

\begin{proof}
	Consider an online MAPF instance on the 4-neighbor 2D grid shown in Figure \ref{fig:all_flowtime_lb} where two agents $a_1$ and $a_2$ are given. Agent $a_1$ has start vertex $v_1$, goal vertex $v_4$, and release time 0. Consider an arbitrary rational online MAPF algorithm. At time step 0, the algorithm computes one of the two possible paths for agent $a_1$ where the agent starts at time step 0, moves without waiting, and arrives at $g_1 = v_4$ at time step 2 according to Definition \ref{dfn:rationality}. If agent $a_1$ chooses to go through vertex $v_2$, we construct agent $a_2$ with start vertex $v_2$, goal vertex $v_1$, and release time 1. Otherwise (agent $a_1$ chooses to go through vertex $v_3$), we construct agent $a_2$ with start vertex $v_3$, goal vertex $v_1$, and release time 1. In either case, agent $a_2$ starts at time step 2. The resulting flowtime is 4. The optimal plan has flowtime 3. The theorem thus follows.
\end{proof}

\begin{thm}\label{thm:all_makespan_lb}
	There exists an online MAPF instance on a 4-neighbor 2D grid for which any rational online MAPF algorithm achieves a competitive ratio of at least 3/2 with respect to makespan .
\end{thm}

\begin{proof}
	Using the same online MAPF instance and argument as in the proof of Theorem \ref{thm:all_flowtime_lb}, any rational online MAPF algorithm results in makespan 3. The optimal plan has makespan 2. The theorem thus follows.
\end{proof}

%\begin{thm}
%	There exists an online MAPF instance on a 4-neighbor 2D grid for which any rational online MAPF algorithm achieves infinite competitive ratio with respect to the latency objective.
%\end{thm}

\subsection{Infinite Competitive Ratio for Latency}\label{sec:infinite_latency}

We have not performed any analysis for latency so far since we have the following observation from the above example.

\begin{obs}\label{obs:latency}
	All rational online MAPF algorithms have infinite competitive ratio with respect to latency even on 4-neighbor 2D grids.
\end{obs}

\begin{proof}[Reason]
	Using the same online MAPF instance and argument as in the proof of Theorem \ref{thm:all_flowtime_lb}, any rational online MAPF algorithm results in latency 1. The optimal plan has latency 0. Therefore, the competitive ratio of any rational online MAPF algorithm is (at least) infinite with respect to latency even on 4-neighbor 2D grids.
\end{proof}

\section{Conclusions}\label{sec:conclusion}

We conducted a theoretical study of online MAPF for the first time.
%Table \ref{tab:summary} summarizes the main results.
Our results suggest that,
%all rational algorithms, including the optimally-rational ones, are asymptotically no more effective than the naive algorithm SEQUENCE
if rerouting is disallowed, then planning for multiple agents is asymptotically (only) as effective as planning for one agent at a time and
acting optimally rationally is asymptotically (only) as effective as acting rationally, which is also asymptotically (only) as effective as following the naive algorithm SEQUENCE. However, allowing rerouting can potentially result in high effectiveness, as indicated by the gap between the competitive ratio upper and lower bounds.

Future work includes (1) developing a rational online MAPF algorithm in PLAN-ALL that achieves the current competitive ratio lower bound (thus proving that the bound is tight for it) or tightening the bounds further and (2) analyzing the online MAPF variant where probabilistic models of future arrivals of agents are given.

%\section*{Acknowledgments}
%This work was supported by the Natural Sciences and Engineering Research Council (NSERC) under grant number RGPIN-2020-06540.

\section*{Appendix: Proof of Theorem \ref{thm:makespan_hardness}}
%The proof of Theorem \ref{thm:makespan_hardness} is as follows.
\begin{proof}%[Proof of Theorem \ref{thm:makespan_hardness}]
	%We follow the reduction used in the proof of Theorem 3 in \citet{MaAAAI16} to construct an online MAPF instance that has a solution with makespan three if and only if a given $\le$3,$=$3-SAT instance is satisfiable.
	Similar to the proof of Theorem 3 in \citet{MaAAAI16}, we construct an online MAPF instance that has a solution with makespan three if and only if a given $\le$3,$=$3-SAT instance is satisfiable.	
	For each variable $X_i$ in the $\le$3,$=$3-SAT instance, we construct two	``literal'' agents, $a_{iT}$ and $a_{iF}$, with start vertices $s_{iT}$ and $s_{iF}$ and goal vertices $t_{iT}$ and $t_{iF}$, respectively. All literal agents have release time zero. For each literal agent, we construct two paths to get to its goal vertex in three time steps: a ``shared'' path, namely $\langle	s_{iT},u_{iT},v_i, t_{iT}\rangle$ for $a_{iT}$ and $\langle
	s_{iF},u_{iF},v_i, t_{iF}\rangle$ for $a_{iF}$, and a ``private'' path,
	namely $\langle s_{iT},w_{iT},x_{iT}, t_{iT}\rangle$ for $a_{iT}$ and
	$\langle s_{iF},w_{iF},x_{iF}, t_{iF}\rangle$ for $a_{iF}$. The shared paths
	for $a_{iT}$ and $a_{iF}$ intersect at vertex $v_i$. Only one of the two
	paths can thus be used if a makespan of three is to be achieved. Sending
	literal agent $a_{iT}$ (or $a_{iF}$) along the shared path corresponds to
	assigning \emph{True} (or \emph{False}) to $X_i$ in the $\le$3,$=$3-SAT instance.		
	For each clause $C_j$ in the $\le$3,$=$3-SAT instance, we construct a
	``clause'' agent $a_j$ with start vertex $c_j$, goal vertex
	$d_j$, and release time zero. It has multiple (but at most three) ``clause'' paths to get to its
	goal vertex in three time steps, which have a one-to-one
	correspondence to the literals in $C_j$. Every literal $X_i$ (or
	$\overline{X_i}$) can appear in at most two clauses. If $C_j$ is the first
	clause that it appears in, then the clause path is $\langle c_j, w_{iT},
	b_j, d_j\rangle$ (or $\langle c_j, w_{iF}, b_j, d_j\rangle$). If $C_j$ is
	the second clause that it appears in, a vertex $\alpha_j$ is introduced and the
	clause path is instead $\langle c_j, \alpha_j, x_{iT}, d_j\rangle$ (or $\langle
	c_j, \alpha_j, x_{iF}, d_j\rangle$). The clause path of each $C_j$ with respect
	to any literal in that clause and the private path of the literal
	intersect. Only one of the two paths can thus be used if a makespan of three
	is to be achieved.	
%	\begin{figure}[t]
%		\centering
%		\includegraphics[width=1.00\columnwidth]{hardness}\\
%		\caption{An online MAPF instance reduced from the $\le$3,$=$3-SAT instance $(X_1
%			\vee X_2 \vee \overline{X}_3) \wedge (\overline{X}_1 \vee X_2 \vee
%			\overline{X}_3) \wedge (X_1 \vee \overline{X}_2 \vee X_3)$. Clause $C_1$
%			is the first clause that literal $X_1$ appears in. The corresponding
%			clause path is $\langle c_1, w_{1T}, b_1, d_1 \rangle$.  Since clause
%			$C_2$ is the second clause that $X_2$ appears in, vertex $\alpha_2$ is
%			introduced. The corresponding clause path is $\langle c_2, \alpha_2, x_{2T},
%			d_2 \rangle$. The blue (directed) edges represent one optimal solution to the online MAPF instance of makespan three, which corresponds to the satisfying
%			assignment $(X_1,X_2,X_3) = ($\emph{False},\emph{True},\emph{True}$)$.}\label{fig:PERR_hardness}
%	\end{figure}
%	Figure~\ref{fig:PERR_hardness} shows an example of the construction.	
	A visualized example of the reduction can be found in \citet{MaAAAI16}.	
	Suppose that a satisfying assignment to the $\le$3,$=$3-SAT instance
	exists. Then, a solution with makespan three is obtained by sending literal
	agents of true literals along their shared paths, the other literal
	agents along their private paths, and clause agents along the clause
	paths corresponding to one of the true literals in those clauses.	
	Conversely, suppose that a solution with makespan three exists. Then, each
	clause agent traverses the clause path corresponding to one of the
	literals in that clause, and the corresponding literal agent traverses its
	shared path. Since the agents of a literal and its complement cannot both
	use their shared path if a makespan of three is to be achieved, we can
	assign \emph{True} to every literal whose agent uses its shared path
	without assigning \emph{True} to both the uncomplemented and complemented
	literals. If the agents of both literals use their private paths, we can
	assign \emph{True} to any one of the literals and \emph{False} to the other
	one. A solution to the online MAPF instance with makespan three thus yields a
	satisfying assignment to the $\le$3,$=$3-SAT instance.	
	
	To summarize, the online MAPF instance has a solution with makespan three if and
	only if the $\le$3,$=$3-SAT instance is satisfiable. Also, the online MAPF instance
	cannot have a solution with makespan smaller than three and always has a
	solution with makespan four, even if the $\le$3,$=$3-SAT instance is
	unsatisfiable. For any $\epsilon>0$, any approximation algorithm for online MAPF
	with ratio $4/3 - \epsilon$ thus computes a solution with makespan three
	whenever the $\le$3,$=$3-SAT instance is satisfiable and therefore solves
	$\le$3,$=$3-SAT.	
\end{proof}

%\section*{Acknowledgments}
%This work was supported by the Natural Sciences and Engineering Research Council (NSERC) under grant number RGPIN-2020-06540.

%\newpage
	\small
	\bibliography{references}
	\bibliographystyle{aaai}
\end{document}